\documentclass[twoside,11pt]{article}

\usepackage{ntheorem}
\usepackage{jmlr2e}
\jmlrheading{17}{2016}{1-30}{12/15; Revised 8/16}{10/16}{Justin Bed\H{o} and Cheng Soon Ong}
\ShortHeadings{Multivariate Spearman's $\rho$ for Aggregating Ranks Using Copulas}{Bed\H{o} and Ong}
\firstpageno{1}

\usepackage{graphicx}
\usepackage{subfig}
\usepackage{booktabs}
\usepackage{microtype}

\usepackage{natbib}
\usepackage{hyperref}
\usepackage{cleveref}
\usepackage{amsmath}

\usepackage{rotating}

\usepackage{algorithm}
\usepackage{algpseudocode}
\algrenewcommand\algorithmicrequire{\textbf{Input:}}
\algrenewcommand\algorithmicensure{\textbf{Output:}}

\newcommand{\ones}{\mathbf{1}}

\newcommand{\RR}{\mathbb{R}}

\newcommand{\dd}{\; \mathrm{d}}
\DeclareMathOperator{\dom}{Dom}
\DeclareMathOperator{\codom}{codom}

\usepackage{color}
\newcommand{\cheng}[1]{}

\title{Multivariate Spearman's $\rho$ for Aggregating Ranks Using Copulas}
\author{\name{Justin Bed\H{o}}\email{cu@cua0.org} \\
  \addr{The Walter and Eliza Hall Institute, 1G Royal Parade, Parkville Victoria
    3052, Australia\\
    The Department of Computing and Information Systems, the
    University of Melbourne, VIC 3010 Australia}
  \AND\name{Cheng Soon Ong}\email{chengsoon.ong@anu.edu.au}\\
  \addr{Data61, CSIRO, 7 London Circuit, Canberra ACT 2601 Australia\\
    Research School of Computer Science, the Australian National
    University, Australia\\
    The Department of Electrical and Electronic Engineering,
    the University of Melbourne, VIC 3010 Australia\\
    }}
\editor{Jie Peng}

\begin{document}

	\maketitle

	\begin{abstract}%
		We study the problem of rank aggregation: given a set of ranked lists,
		we want to form a consensus ranking.
		Furthermore, we consider the case of extreme lists: i.e., only the
		rank of the best or worst elements are known.
		We impute missing ranks and generalise
		Spearman's $\rho$ to extreme ranks.
		Our main contribution is the
		derivation of a non-parametric estimator
		for rank aggregation based on multivariate extensions of Spearman's $\rho$,
		which measures correlation between a set of ranked lists.
		Multivariate Spearman's $\rho$ is defined using copulas, and we show that
		the geometric mean of normalised ranks maximises multivariate
		correlation.
		Motivated by this, we propose a weighted geometric mean approach for
		learning to rank which has a closed form least squares solution.
		When only the best (top-k) or worst (bottom-k) elements of a
		ranked list are known, we
		impute the missing ranks by the average value, allowing us to apply
		Spearman's $\rho$.
		We discuss an optimistic and pessimistic imputation
		of missing values, which respectively maximise and minimise correlation,
		and show its effect on aggregating university rankings.
		Finally, we demonstrate good performance on the rank aggregation
		benchmarks MQ2007 and MQ2008.
	\end{abstract}

	\section{Introduction}

	Ranking is a central task in many applications such as information
	retrieval, recommender systems and bioinformatics.
	It may also be a subtask of other learning problems such as
	feature selection, where features
	are scored according to their predictiveness, and then the most
	significant ones are selected.
	One major advantage of ranks over scores is that the resulting
	predicted ranks are automatically normalised and hence can be used to
	combine diverse sources of information.
	However, unlike many other supervised learning problems, the problem
	of learning to rank~\citep{lebanon08nonpmp,liu08learir} does not have
	the simple one example one label paradigm.
	This has led to many formulations of learning tasks, depending on what
	label information is available, including pairwise ranking, listwise
	ranking and rank aggregation.

	This paper considers a novel formulation of rank aggregation
	based on multivariate extensions to Spearman's $\rho$.
	For a set of $n$ objects from the domain $\Omega$, we are given a set of
	$d$ experts that rank these objects providing rankings $R_1,\ldots,R_d$.
	Each rank is a permutation of the $n$ objects, and can be represented as a
	vector of unique integers from 1 to $n$.
	The problem of rank
	aggregation is to construct a new vector $R$ that is most similar to
	the set of $d$ ranks provided by the experts.
	In this paper we use Spearman's correlation $\rho$, a widely used
	correlation measure for ranks~\cite{spearman04promat}.
	Instead of decomposing the association into a combination of pairwise
	similarities, $\rho(R, R_1),\rho(R, R_2),\ldots,\rho(R, R_d)$, we
	directly maximise the multivariate correlation
	\[
		R^* = \arg\max_R \rho(R, R_1, R_2,\ldots, R_d).
	\]

	Measures of association such as Spearman's $\rho$ capture the
	concordance between random variables~\citep{nelsen06intc}.
	Informally,
	random variables are concordant if large values of one tend to be
	associated with large values of the other.
	Let $(x_i,y_i)$ and
	$(x_j,y_j)$ be two observations of a pair of continuous
	random variables.
	We say that $(x_i,y_i)$ and $(x_j,y_j)$ are
	\emph{concordant} if $x_i<x_j$ and $y_i<y_j$ or if $x_i>x_j$ and
	$y_i>y_j$.
	If the inequalities disagree, we say that the samples are
	\emph{discordant}.
	The concept of concordance captures only the order
	of the random variables, and is invariant to their values, and
	therefore is ideal for analysing ranks.
	As will be described in \cref{sec:multi-spearman},
	Spearman's $\rho$ is based on the
	difference between the concordance and discordance of the samples.

	In short, Spearman's correlation can be defined as the concordance $Q$
	between the copula $C$ corresponding to the data and the independent
	copula $\pi$
	\[
		\rho \propto Q(C,\pi).
	\]
	We review the concept of copulas in \cref{sec:copulas} and
	derive our generalisation of concordance in \cref{sec:spearman}.
	While the mathematical machinery to derive our proposed algorithm
	relies on constructions that may not be familiar to some machine
	learners, the resulting algorithm for rank aggregation is
	straightforward.
	We solve a least squares problem for $n$ items,
	\[
		\min_\omega \sum_{x=1}^n{\left(l(x) - \sum_{j=1}^d\omega_d r_d(x)\right)}^2,
	\]
	where we minimise the weights $\omega_1,\dots,\omega_d$ corresponding
	to the $d$ experts.
	It turns out that the appropriate transformation to learn weights
	between experts is to use logarithmic scaled ranks.
	In the above equation, $l(x)$ and $r(x)$ denote the logarithm of the
	labels and individual expert ranks respectively, with all ranks
	normalised uniformly to the interval $(0,1)$.
	Since it is a least squares problem, there is a closed form solution
	for the optimal weights.
	This is in contrast to previous approaches to rank aggregation that
	involve complex optimisation methods or sampling.

	\subsection{Our Contributions}

	We theoretically justify why the above
	least squares problem provides a meaningful way to weight experts.
	We show that the geometric mean of a set of normalised ranks maximises
	multivariate Spearman's $\rho$.
	This motivates our method which finds a setting of weights that
	maximise multivariate Spearman's $\rho$ for a specific target (supervised
	rank aggregation).

	As previously mentioned, in many applications of rank aggregation,
	only extreme ranks are available, whereas the standard definitions of
	Spearman's $\rho$ require full ranks.
	For practical problems, the expert may only rank the most liked
	(top-$k$) or most disliked (bottom-$k$) objects where $k$ can be
	different for each expert.
	We propose a method for estimating Spearman's $\rho$ for extreme ranks by
	imputing the remaining ranks.
	We describe this method and show that it is an unbiased estimator in
	\cref{sec:rank-imputation}.

	This results in a non-parametric approach for rank aggregation that
	learns the weights of experts by solving a least squares problem.
	The weights in this case model dependencies between the rankings, i.e., the
	rankings are not independent.
	This is different to much prior work (see \cref{sec:related}) in that we
	explicitly learn the dependencies between experts simultaneously and not in a
	pairwise fashion.
	Our method thus offers significant computational benefits, modelling
	flexibility in the presence of dependencies between experts, and also
	interpretability due to the simplicity of the model.
	In \cref{sec:benchmarks} we describe our empirical results for rank
	aggregation and show that our simple algorithm performs better than
	current state of the art results.

	\subsection{Multiple Representations of Ranks}

	There are a wide range of applications which benefit from rank analysis,
	resulting in various equivalent ways to represent ranks and orderings.
	The basic representation often used in introductory texts is to provide
	the list of objects, for example $[a, b, c, d, e, f]$, denoting the fact
	that $a$ is the most highly ranked object and $f$ is the lowest ranked.
	It is often more convenient to numerically represent the rank for computational
	purposes, that is to keep a list of integers $1,\dots, n$ corresponding
	to the rank of a particular object.
	For the example above, by maintaining
	the set of objects as is, the ranks are then $[1, 2, 3, 4, 5, 6]$.
	It turns out for empirical copula modeling, it is important that the
	numerical values are in the interval $(0,1)$, and therefore we normalise
	the numerical representation by $n+1$, that is
	$[\frac{1}{7}, \frac{2}{7}, \frac{3}{7}, \frac{4}{7}, \frac{5}{7},
	\frac{6}{7}]$.
	However, note that the numerical representation is actually dependent on the
	fact that we have maintained the set of objects in a particular fashion.
	In fact, by the above numerical list, we are saying that object $a$ has rank
	$\frac{1}{7}$,
	and object $f$ has rank $\frac{6}{7}$.
	In other words, we are defining a permutation
	mapping $R: \Omega \to (0,1)$ from the space of objects $\Omega$ to the interval
	$(0,1)$.

	\subsection{Related Work}\label{sec:related}

	There are two related rank aggregation tasks: score based rank
	aggregation and order based rank aggregation.
	For score based rank aggregation objects are associated with scores,
	while for order based rank aggregation only the relative order of
	objects are available.
	There has been recent work on combining both scores and
	ranks~\citep{sculley10comrr,iyer13lovbdc}.
	We consider the learning task referred to as the listwise approach in
	\citet{liu08learir}, where the input is a set of ranked lists of
	documents from multiple experts, and the learner has to predict the
	final ranks.
	Numerous proposals for solving the problem of combining multiple lists
	into a single list are surveyed in \citet{liu08learir}.
	\citet{Niu2012} has focused on learning a good ranking from given
	features.
	A good review of probability and statistics applied to permutations is
	\citet{diaconis88grorps}.

	Spearman's $\rho$ is a natural measure of similarity for distributions of
	permutations~\citep{mallows57nonnrm,fligner86disbrm}.
	Interestingly, there has not been much work using Spearman's $\rho$ for
	dealing with ranked data, but instead the focus has been on Kendall's
	$\tau$.
	One difficulty of inference with the Mallows model~\citep{mallows57nonnrm} for
	Spearman's $\rho$ is that it involves estimating the permanent of a matrix.
	Our model is derived from the copula form of Spearman's $\rho$ and allows a simple
	formulation for aggregation that does not require any computationally complex
	operations, thus providing a significant computational advantage.

	Other previous approaches~\citep{klementiev08unsrad,iyer12subbld} to rank
	aggregation considers pairwise comparisons between ranked lists.
	In contrast, our approach does not consider pairwise combinations and operates
	over all lists.
	We prove a result saying that the geometric mean of normalised ranks
	maximise Spearman's $\rho$ (\cref{prop:rho-max-mean}), which is similar in
	spirit to the result in \citet{iyer12subbld} that shows that for
	Lov\'asz--Bregman divergences the best aggregator is the arithmetic mean.
	This provides a computational advantage over pairwise methods as the number of
	lists grows.

	Our work builds heavily on copula theory, and we use results from
	\citet{nelsen06intc}.
	Brief introductions to copulas can be found in
	\citet{trivedi05copmip}, \citet{genest07eveyaw}, and
	\citet{elidan13copml}.
	Further details on copula modeling are available
	in a recent book~\citep{joe14}.
	Many of these results are presented for
	bivariate copulas only.
	There are fewer results on multivariate
	copulas~\citep{joe90mulc,nelsen96nonmma}
	and their relation to
	Spearman's $\rho$~\citep{ubeda-flores05mulvbb,schmid10copbmm},
	which we shall discuss later in this paper.

	Finally, other well known measures of bivariate dependence have forms under
	the copula framework and have multivariate extensions.
	In particular, multivariate extensions of Kendall's $\tau$ have been
	proposed~\citep{joe14}.
	It is possible investigations into these copula formulations results in other
	efficient aggregation methods with different tradeoffs, however in this work
	we focus on Spearman's $\rho$.

	The work on partial ranks goes back to at least
	\citet{critchlow85metmap}, who describes the rank aggregation task in
	terms of distances between rankings.
	We have applied the results
	of this paper to rank aggregation~\citep{macintyre14assdrg} and
	stability estimation~\citep{bedo14stabgb} in the domain of life sciences.

	\section{Copulas}\label{sec:copulas}

	Copulas are functions from the unit hypercube to the unit
	interval~\citep{elidan13copml}.
	In this section we briefly
	review the bivariate setting, in preparation for the multivariate setting in the
	next section.
	The expert reader may skip directly to \cref{sec:spearman} to see the definition
	of
	multivariate Spearman's $\rho$ in terms of the multivariate copula.

	\subsection{Definition of Copulas}\label{sec:copula-bivariate}
	Intuitively, for continuous random variables copulas model the dependence
	component of a
	multivariate distribution after discounting for univariate marginal effects.
	We let $\RR$ denote the ordinary real line $(-\infty,\infty)$, and
	$\overline{\RR}$ denote the extended real line $[-\infty,\infty]$.
	The following algebraic definition of bivariate copulas is generalised to the
	multivariate setting
	in \cref{sec:spearman}.
	It essentially constrains copulas to be functions that are
	{\em monotonically increasing\/} along each dimension as well as towards
	the diagonal of the volume.

	\begin{definition}\label{def:h-volume-bivariate}
		Let $A_1$ and $A_2$ be nonempty subsets of $\overline{\RR}$, and let
		$H(\cdot,\cdot)$ be a real function such that the domain of
		$H=A_1\times A_2$.
		Let $B=[x_1,x_2]\times[y_1,y_2]$ be a rectangle
		all of whose vertices are in the domain of $H$.
		Then the
		\emph{$H$-volume of $B$} is given by:
		\[
			V_H(B) = H(x_1,y_1) + H(x_2,y_2) - H(x_1,y_2) - H(x_2,y_1).
		\]
	\end{definition}

	\begin{definition}
		A real function $H(\cdot,\cdot)$ is \emph{2-increasing} if its $H$-volume is
		non-negative,
		that is $V_H(B) \geqslant 0$ for all rectangles $B$ whose vertices lie in the domain
		of $H$.
	\end{definition}

	\begin{definition}
		A \emph{copula} is a function $C\colon {[0,1]}^2 \to [0,1]$ with the
		following properties:
		\begin{enumerate}
			\item For every $u,v\in[0,1]$,
			\[
				C(u,0) = 0 = C(0,v)
			\]
			\[
				C(u,1) = u\quad\textrm{and}\quad C(1,v) = v
			\]
			\item $C$ is 2-increasing.
		\end{enumerate}
	\end{definition}

	\subsection{Relation Between Bivariate Cumulative Density Functions and
	  Copulas}

	Sklar's theorem is central to the theory of copulas and is the
	foundation of many applications in statistics.
	Indeed, Sklar's
	theorem can be defined for general distribution functions outside
	of probabilistic settings.
	However, since we are interested in
	statistical applications we will consider cumulative distribution
	functions.

	\begin{theorem}[Sklar's theorem]

		Let $H(\cdot, \cdot)$ be a cumulative distribution function with
		marginals $F(\cdot)$ and $G(\cdot)$.
		Then there exists a copula
		$C \colon {[0,1]}^2 \rightarrow [0,1]$
		such that for all $x,y$ in $\overline{\RR}$,
		\[
			H(x,y) = C(F(x),G(y)).
		\]
		If $F(\cdot)$ and $G(\cdot)$ are continuous then $C(\cdot,\cdot)$ is unique;
		otherwise $C(\cdot,\cdot)$ is uniquely determined on the ranges of $F(\cdot)$
		and
		$G(\cdot)$.

		Conversely, if $C(\cdot,\cdot)$ is a copula and $F(\cdot)$ and $G(\cdot)$ are
		cumulative
		distribution functions then the function $H(\cdot,\cdot)$ is a bivariate
		cumulative distribution function with marginals $F(\cdot)$ and $G(\cdot)$.
	\end{theorem}

	\section{Spearman's $\rho$}\label{sec:spearman}

	We briefly review
	the bivariate model to lay out the approach for estimating the copula using data,
	the so-called
	empirical copula.

	\subsection{Empirical bivariate Spearman's
	  $\rho$}\label{sec:bivariate-spearman}
	Let $R$ and $S$ be ranking functions, which are bijections mapping
	elements $x$ in the domain $U$
	to $[1,2,\dots,n]$.
	The domain $U$ represents the space of objects that we are interested in ranking,
	such as documents
	retrieved in response to a query or the biomarkers most associated
	with a disease.
	Since we consider only the ranks of the object $R(x)$
	and $S(x)$, the
	actual domain $U$ does not affect the analysis.
	The sums below are
	over the $n$ objects $x$.
	Similar to the approach of Pearson's correlation for the measure of
	dependence, Spearman's $\rho$ is a measure of correlation between
	ranks, empirically given by:
	\begin{equation}
		\label{eq:emp-spearman}
		\rho_n = \frac{\sum_{x} (R(x) - \bar{R}) (S(x) - \bar{S}) }
		{\sqrt{\sum_{x} {(R(x) - \bar{R})}^2 \sum_{x} {(S(x) - \bar{S})}^2}},
	\end{equation}
	where $\bar{R} := \frac{1}{n}\sum_{x} R(x)$ and $\bar{S} :=
	\frac{1}{n}\sum_{x} S(x)$ are the
	empirical means of the respective random variables.
	This is equivalent to applying Pearson's correlation to the ranks instead of the
	values of the score
	function itself.
	There is no direct way to generalise this expression to more than two
	ranking functions, but as we shall see in
	\cref{sec:multi-spearman} we can obtain an expression via the copula.

	By substituting the definitions of the empirical means and rearranging the terms,
	we obtain
	\[
		\rho_n = \left(\frac{n+1}{n-1}\right)
		\left[
		\frac{12}{n}\sum_{x} \frac{R(x)}{n+1}\frac{S(x)}{n+1} - 3
		\right].
	\]
	The constants 12 and 3 seem strange, but are a natural consequence of the
	mean and variance of a list of ranks.
	As we will see later, these constants are dependent only on
	the dimension of the copula.
	Similar to the definition of an
	empirical CDF, we define an empirical copula as:
	\[
		C_n(u,v) = \frac{1}{n} \sum_{x} \ones\left(
		\frac{R(x)}{n+1}\leqslant u, \frac{S(x)}{n+1}\leqslant v
		\right),
	\]
	where $\ones$ is the indicator function.
	This allows us to re-express the form of $\rho_n$ above in terms of an
	integral over the unit square,
	\[
		\rho_n = \left(\frac{n+1}{n-1}\right)
		\left[
		\frac{12}{n}\sum_{x} \frac{R(x)}{n+1}\frac{S(x)}{n+1} - 3
		\right]
		= \left(\frac{n+1}{n-1}\right)
		\left[
		12 \int_{{[0,1]}^2} uv C_n(u,v) - 3
		\right].
	\]
	It can be shown \citep{nelsen06intc,genest07eveyaw}
	that $\rho_n$ is an asymptotically unbiased estimator of
	\[
		\rho = 12 \int_{{[0,1]}^2} C(u,v) \dd u \dd v - 3,
	\]
	where $C$ is the population version of $C_n$.

	\subsection{Multivariate Copulas}

	We now generalise the definitions in \cref{sec:copula-bivariate} to the
	multivariate
	case.
	The concepts are essentially the same, constraining the copula to be
	``monotonically
	increasing'' in the interval $[0,1]$ and also towards the center of the
	volume~\citep{durante10copti}.

	\begin{definition}
		Let $A_j$ be nonempty subsets of $\overline{\RR}$ for $j=1,\dots,d$, and let
		$H_d\colon A_1\times\cdots\times A_d \to \RR$.
		Let $B=[a_1,b_1]\times\cdots\times[a_d,b_d]$ be the $d$-box where all
		vertices are contained in $\dom H_d$.
		Then the
		\emph{$H_d$-volume of $B$} is the $d^\text{th}$ order difference:
		\[
			V_{H_d}(B) = \Delta_{a_d}^{b_d}\dots\Delta_{a_1}^{b_1}H_d(\vec{t}),
		\]
		where
		\begin{align*}
			\Delta_{a_i}^{b_i}H(\vec{t}) = &H_d(t_1,\dots,t_{i-1},b_i,t_{i+1},\dots,t_d)\\
			&- H_d(t_1,\dots,t_{i-1},a_i,t_{i+1},\dots,t_d).
		\end{align*}
	\end{definition}

	\begin{definition}
		A real function $H_d$ is \emph{grounded} if $H_d(\vec{t})=0$ for all $t\in\dom
		H_d$ such that $t_j = a_j$ for at least one $j\in\{1,\ldots,d\}$.
	\end{definition}

	\begin{definition}
		A real function $H_d$ is \emph{d-increasing} if $V_{H_d}(B)
		\geqslant 0$ for all n-boxes $B$ whose vertices lie in the domain
		of $H$.
	\end{definition}

	\begin{definition}\label{def:copula-multivariate}
		A multivariate \emph{copula} has the following properties:
		\begin{enumerate}
			\item $\dom C = {[0,1]}^d $
			\item $C$ has margins $C_j(u)=C(1,\dots,1,u,1,\dots,1)=u$ for all $j$ and
			$u\in I$
			\item $C$ is grounded
			\item $C$ is d-increasing.
		\end{enumerate}
	\end{definition}

	There is an alternative probabilistic definition that may be more familiar to
	readers with a
	statistical background.
	\begin{definition}
		Let $U_1,\dots,U_d$ be real uniformly distributed random variables
		on the unit interval $\sim U([0,1])$.
		A copula function $C\colon {[0,1]}^d \longrightarrow [0,1]$ is a joint
		distribution
		\[
			C(u_1,\dots,u_d) = P(U_1\leqslant u_1,\dots,U_d\leqslant
			u_d).
		\]
	\end{definition}

	Let $X\sim F$ be a continuous random variable such that the inverse of the CDF
	$F^{-1}$ exists.
	What is the distribution of $F(x) = P(X\leqslant x)$?
	\begin{align*}
		P(F(X)\leqslant u) &= P(F^{-1}(F(X)) \leqslant F^{-1}(u))\\
		&= P(X \leqslant F^{-1}(u))\\
		&= F(F^{-1}(u)) = u
	\end{align*}
	The above calculation shows that the distribution is uniform, i.e.
	$F(x) \sim U([0,1])$.
	This can be considered to be the {\em copula trick}, as the user has the freedom to
	choose the
	copula independently of the marginal distributions.

	\subsection{Multivariate Extension of Spearman's
	  $\rho$}\label{sec:multi-spearman}

	We generalise the concept of concordance to the multivariate setting such that we
	can define
	multivariate Spearman's $\rho$ in an analogous way to the bivariate $\rho$ as
	defined in \citet{nelsen06intc}.

	Recall that two random variables are concordant if they tend to be in
	the same order, that is
	$(x_i,y_i)$ and $(x_j,y_j)$ are concordant if
	$(x_i-x_j)(y_i-y_j)>0$,
	and are discordant if
	$(x_i-x_j)(y_i-y_j)<0$.
	The concordance function $Q$ denotes the
	difference between the probabilities of concordance and discordance,
	and as the following theorem shows, can be expressed in terms of the copulas.
	The proof is in~\citet{nelsen06intc}.

	\begin{theorem}[Concordance function]\label{th:bivariateQ}
		Let $(X_1, Y_1)$ and $(X_2, Y_2)$ be two independent vectors with joint
		distributions $H_1(x, y)=C_1(F(x),G(y))$ and $H_2(x, y) = C_2(F(x), G(y))$
		respectively.
		Then the concordance function $Q$ is given by
		\begin{align*}
			Q(C_1, C_2) :=& P[(X_1-X_2)(Y_1-Y_2)>0]
			-P[(X_1-X_2)(Y_1-Y_2)<0] \\
			=& 4 \int_{{[0,1]}^2}C_2(u,v) \dd C_1(u,v)-1
		\end{align*}
	\end{theorem}

	We now state the generalisation of concordance to the multivariate
	case~\citep{nelsen96nonmma,joe90mulc}.
	Further details of multivariate concordance can be found in
	\citet{taylor07mulmc} and \citet{schmid10copbmm}.

	\begin{definition}[Multivariate concordance]\label{def:mvconcord}
		Let $(X_1,\ldots,X_d)$ and $(Y_1,\ldots,Y_d)$ be two independent $d$-vectors
		with
		joint distributions $C_X(F(x))$ and
		$C_Y(F(y))$ where $F(x) = F_1(x_1),\ldots,F_d(x_d)$
		and $F(y) = F_1(y_1),\ldots,F_d(y_d)$ are the marginal
		distributions, and $C_X,C_Y$ are the respective $d$ copulas.
		Then the concordance function $Q$ is given by
		\begin{align*}
			Q(C_X, C_Y) :=& 2^d\int_{{[0,1]}^d}C_X(u) \dd C_Y(u)-1.
		\end{align*}
	\end{definition}

	Note that although the integral is a straight forward generalisation of
	\cref{th:bivariateQ}, it is no-longer equal to the difference between the
	probability of concordance and discordance.
	Consequently, the properties
	possessed by $Q$ are different for $d>2$.

	There are three copulas that are of
	particular interest: the independent copula $\pi(u) := \prod_i u_i$,
	and the upper and lower Fr\'echet--Hoeffding bounds,
	$M(u) = \min \{ u_1, u_2, \dots, u_d\}$ and $W(u) \geq \max
	\{u_1+u_2+\cdots+u_d-(d-1),0\}$ respectively~\citep[pg.~48]{joe14}.
	Note that that while $W$ is point-wise sharp, this lower
	bound is not itself a copula, and hence the lower bound is
	not tight~\citep{ubeda-flores05mulvbb}.

	\begin{theorem}\label{Qprops}
		Let $C$, $C'$, and $Q$ be given as in \cref{def:mvconcord}, $M$ and $W$ be the
		upper and lower Fr\'echet--Hoeffding bounds respectively, and assume $d>2$.
		Then
		\begin{enumerate}
			\item $Q$ is symmetric in its arguments if $C=C'$.
			\item $Q$ is non-decreasing in the first argument, and both arguments if
			$C=C'$.
			\item $-1\leq Q(W,W) \leq Q(C,C) \leq Q(M,M) = 2^{d-1}-1$.
			\item $Q(\pi,\pi) = 0$.
		\end{enumerate}
	\end{theorem}

	\begin{proof}
		\paragraph{Property 1}
		The first property is clear from the definition of $Q(C,C')$ and
		the properties of integration.

		\paragraph{Property 2}
		Q is non-decreasing in the first argument by properties of
		integration.
		For the second part, notice that
		\begin{align*}
			\int C(u) \dd C(u) =& C^2(u) - \int C(u) \dd C(u)\\
			\Rightarrow \int C(u) \dd C(u) =& \frac12 C^2(u)
		\end{align*}
		by applying integration by parts.
		The property now follows.

		\paragraph{Property 3}
		It follows that
		\begin{align*}
			Q(M,M) =& 2^d\int_{{[0,1]}^d} M(u) \dd M(u) - 1\\
			=& 2^d\int_0^1u \dd u - 1\\
			=& 2^{d-1}-1,
		\end{align*}
		and
		\begin{align*}
			Q(W,W) =& 2^d\int_{{[0,1]}^d} W(u) \dd W(u) - 1\\
			\geq& 2^d\int_0^1 0 \dd u - 1\\
			=& -1.
		\end{align*}
		Property 3 now follows from the first two properties.

		\paragraph{Property 4}
		\begin{align*}
			Q(\pi,\pi) =& 2^d\int_{{[0,1]}^d} \pi(u) \dd \pi(u) - 1\\
			=& 2^d\int_{{[0,1]}^d} u \dd u - 1\\
			=& \frac{2^d}{2^d}-1\\
			=& 0
		\end{align*}

	\end{proof}

	It is clear from this theorem that $Q$ is well calibrated at
	$Q(W,W)$ and $Q(\pi,\pi)$, however not for $Q(M,M)$.
	Consequently, with this multidimensional extension it becomes
	increasingly difficult to estimate discordance as $d$ increases.

	\begin{proposition}\label{prop:Qmpi}
		Let $Q$ be given as in \cref{def:mvconcord}, and $M$ and $\pi$ be the upper
		Fr\'echet--Hoeffding bound and the independent copula respectively, then
		\begin{equation}
			\label{eq:Qmpi}
			Q(M,\pi) = Q(\pi,M) = \frac{2^d - (d+1)}{d+1}.
		\end{equation}
	\end{proposition}

	\begin{proof}
		To show the symmetry,
		\begin{align*}
			Q(M,\pi) =& 2^d\int_{{[0,1]}^d} M(u) \dd \pi(u) - 1\\
			=& 2^d\int_{{[0,1]}^d} u_1u_2\cdots u_d \dd u - 1\\
		\end{align*}
		and
		\begin{align*}
			Q(\pi,M) =& 2^d\int_{{[0,1]}^d} \pi(u) \dd M(u) - 1\\
			=& 2^d\int_{{[0,1]}^d} u_1u_2\cdots u_d \dd u - 1.\\
		\end{align*}
		To obtain the second equality, we observe that
		\begin{align*}
			\int_{{[0,1]}^d} u_1u_2\cdots u_d \dd u
			=& \int_0^1 u^d \dd u\\
			=& \left.
			\frac{1}{d+1} u^{d+1} \right|_0^1\\
			=& \frac{1}{d+1},
		\end{align*}
		and therefore the expression for $Q(M,\pi)$ follows.
	\end{proof}

	In terms of the concordance function, Spearman's $\rho$ is given by the
	concordance between the copula $C$ and the independent copula $\pi(u)
	:= \prod_i u_i$.
	However, unlike the symmetry in~\cref{prop:Qmpi},
	the concordance function is in general not symmetric with respect to its
	arguments.
	This gives us two possible ways of defining multivariate Spearman's $\rho$,
	corresponding
	to $Q(C,\pi)$ and $Q(\pi,C)$.
	Both generalisations are equivalent in the bivariate case,
	and has been called $\rho_d^-$ and $\rho_d^+$ by \citet{nelsen96nonmma} and
	$\rho_1$ and $\rho_2$ by \citet{schmid07mulesr} respectively.
	Naturally, there is a third
	symmetric generalisation which is the average of them.

	\begin{definition}[Multivariate Spearman's $\rho$]
		\begin{align}
			\rho_d^- = h(d) Q(\pi,C) &= h(d)\left[2^d\int_{{[0,1]}^d}C(u) \dd u - 1\right]\label{eq:rho1}
		\end{align}
		and
		\begin{align}
			\rho_d^+ = h(d) Q(C, \pi) = h(d)\left[2^d\int_{{[0,1]}^d}\pi(u) \dd C(u)-1\right],\label{eq:rho2}
		\end{align}
		where $h(d)=\frac{d+1}{2^d - (d+1)}$ is the normalisation factor.
	\end{definition}

	The scaling factor $h(d)$ is derived such that the maximum correlation
	is 1.
	Thus, for Spearman's $\rho$, this is the concordance between the maximum
	copula $M$ and the independent copula $\pi$, which we obtain by \cref{prop:Qmpi}:
	\begin{equation}
		h(d) = 1/Q(M, \pi)
		= \frac{d+1}{2^d - (d+1)}.\label{eq:hd}
	\end{equation}

	Spearman's correlation can equivalently be seen as measuring average orthant
	dependence,
	and the two versions $\rho_d^+$ and $\rho_d^-$ correspond to whether we look at the
	upper
	or lower orthant~\citep{nelsen96nonmma}.
	Positive upper orthant dependence is defined as
	\[
		P(X>x) \geq \prod_{i=1}^d P(X_i > x_i),
	\]
	and positive lower orthant dependence is defined as
	\[
		P(X \leq x) \geq \prod_{i=1}^d P(X_i \leq x_i).
	\]
	When $d=2$, the two definitions are the same and are called positive quadrant
	dependence~\citep{lehmann66somcd},
	as we have already observed for the concordance function:
	\begin{align*}
		P(X_1>x_1, X_2>x_2) &\geq P(X_1 > x_1)P(X_2>x_2)\\
		&\geq [1-P(X_1\leq x_1)][1-P(X_2\leq x_2)]\\
		&\geq 1 - P(X_1\leq x_1) - P(X_2\leq x_2) + P(X_1\leq x_1)P(X_2 \leq x_2).
	\end{align*}
	Rearranging gives
	\[
		P(X_1>x_1, X_2>x_2) + P(X_1\leq x_1) + P(X_2\leq x_2) - 1 \geq P(X_1\leq x_1)P(X_2 \leq x_2).
	\]
	The left hand side is $P(X_1\leq x_1, X_2\leq x_2)$.

	Observe that the scaling factor $h(d)$ is the same for both $\rho^-_d$ and
	$\rho^+_d$ due to \cref{prop:Qmpi}.
	Furthermore, since $P(X_i > x_i) = 1-P(X_i \leq x_i)$ for
	each random variable, the two versions of Spearman's $\rho$ correspond to looking
	at
	whether we
	interpret the ranks as top down or bottom up.
	Converting from one version to the other
	can be done by reinterpreting the data.
	For a particular application, the choice of which
	version to use depends on the ranks that are available.
	We will focus on $\rho_d^+$ henceforth.

	Recall that for a set of $n$ objects from the domain $\Omega$, we are given
	a set of $d$ experts that rank these objects providing ranks
	$R_1,\ldots,R_d$, where each $R_j$ is a bijection to $(0,1)$.
	Putting~\eqref{eq:rho2} and~\eqref{eq:hd} together, we obtain the
	following expression for multivariate Spearman's correlation:
	\begin{equation}
		\label{eq:multi-spearman}
		\rho(R_1, \ldots, R_d) = h(d)Q(C, \pi) = \frac{d+1}{2^d - (d+1)}
		\left[2^d\int_{{[0,1]}^d}\pi(u) \dd C(u) - 1\right].
	\end{equation}
	In practice, we do not have access to the population version of the
	copula $C(u)$ but have the empirical copula $C_n(u)$.
	We discuss this further in \cref{sec:empirical-copula}.

	Unlike the bivariate case, as the number of dimensions increases, the lower bound
	of Spearman's $\rho$
	tends to zero.
	This counterintuitive fact can be understood by considering the three
	dimensional case.
	Consider three rankings $R_1$, $R_2$, and $R_3$.
	If $R_1$ and $R_2$ are anti-correlated ($\rho$=-1),
	and at the same time $R_1$ and $R_3$ are also anti-correlated, this implies that
	$R_2$ and $R_3$
	must be perfectly correlated ($\rho$=1).
	Hence, the overall 3 dimensional correlation is no longer -1.
	This can be made precise by considering the inclusion-exclusion principle, which
	results in
	the following relation from~\citet{nelsen96nonmma}:
	\[
		\frac{1}{2}(\rho_d^-(R_1, R_2, R_3) + \rho_d^+(R_1, R_2, R_3)) =
		\frac{1}{3}(\rho(R_1,R_2)+\rho(R_1,R_3)+\rho(R_2,R_3)).
	\]
	The following corollary defines the lower bound as the number of dimensions
	increases.

	\begin{corollary}\label{th:minW}
		Under the minimum Fr\'echet--Hoeffding bound $W$,
		$Q(W,\pi) \geq -1$ and
		\[
			\lim_{d \to \infty} \rho(R_1,\dots,R_d) \geqslant h(d)Q(W, \pi) = 0.
		\]
		In particular, for dimension $d$,
		\[
			\rho(R_1, \dots, R_d) \geqslant \frac{2^n - (n+1)!}{n!(2^n-(n+1))}.
		\]
	\end{corollary}
	\begin{proof}
		This follows immediately from the bound $-1\leq
		Q(W,\pi)\leq 0$ (from \cref{Qprops}) since $h(d)$ goes to zero as $d\to\infty$.
		The lower bound has also been observed
		in~\citet{nelsen96nonmma} and \citet{schmid10copbmm}.
	\end{proof}

	In summary, the multivariate extension of Spearman's correlation is still
	calibrated
	under maximum correlation as it achieves a value of 1, but it becomes increasingly
	difficult to
	observe anti-correlated sets of ranks as the number of lists to be aggregated
	increases.
	In the next section, we investigate an aggregation algorithm that maximises
	correlation.
	The effect of the lower bound is discussed with respect to imputing missing
	values in \cref{sec:empirical-copula}.

	\section{Optimal Aggregation with Spearman's
	  $\rho$}\label{sec:rank-imputation}

	The empirical copula requires $R$ and $S$ to comprise of ranks for the
	same set of elements, that is $\dom R = \dom S$.
	Recall from \cref{sec:bivariate-spearman} that ranks map to the
	range $\{1,\ldots,n\}$, but the empirical copula is expressed in terms
	of fractional ranks (divided by $n+1$).
	In the following it is convenient to work with normalised ranks, that
	is to consider $R$ and $S$ as bijections to $(0,1)$.
	The expression for the empirical copula then simplifies to
	\begin{equation}
		\label{eq:empirical-copula-2d}
		C_n(u,v) = \frac{1}{|\Omega|} \sum_{x\in \Omega} \ones\left(
		R(x)\leqslant u, S(x)\leqslant v
		\right),
	\end{equation}
	where Ω is the domain of the objects we are interested in ranking.
	Correspondingly, the $d$ dimensional empirical copula for $n$ objects
	given by
	\begin{equation}
		\label{eq:empirical-copula}
		C_n(u) = \frac{1}{n}\sum_{x} \prod_{j=1}^d
		\ones\left( R_j(x)\leqslant u_j \right),
	\end{equation}
	where $R_1(x),\ldots,R_d(x)$ is the rankings of the $d$ experts.
	Plugging the empirical copula~\eqref{eq:empirical-copula} expression
	into Spearman's $\rho$~\eqref{eq:multi-spearman}, and observing that
	integrating the product over the copula is the product of the
	ranks~\cite{schmid07mulesr}, we obtain an empirical expression for
	multivariate Spearman's correlation:
	\begin{equation}
		\label{eq:multi-spearman-emp}
		\rho_n(R_1, \ldots, R_d) = h(d)\left[
		\frac{2^d}{n} \sum_{x} \prod_{j=1}^d R_j(x) - 1
		\right].
	\end{equation}

	\subsection{Geometric Mean is Optimal}\label{sec:geo-mean-opt}

	We are now in a position to derive the deceptively simple result:
	the ranking $R$ that maximises correlation with a given set of rankings
	$\{R_1,\ldots,R_d\}$ is given by the geometric mean of
	$R_1,\ldots,R_d$.
	The following definition is needed to capture the
	notion that ranks only depend on the order.

	\begin{definition}[Rank generator]
		$\sigma \colon \RR^{|\Omega|} \to {[0,1]}^{|\Omega|}$ is a {\em rank
		  generator\/} if:
		\begin{itemize}
			\item for all $x,y \in \Omega$ and $R$ with domain $\Omega$, $R(x) < R(y)\iff
			\sigma\circ R(x) <
			\sigma\circ R(y)$;
			\item for any rankings $R,R'$ with domain $\Omega$ there exists a permutation
			$\xi$ such
			that $\sigma\circ R'=\sigma\circ\xi\circ R$;
			\item for any permutation $\xi$, $\xi\circ\sigma = \sigma \circ \xi$.
		\end{itemize}
	\end{definition}

	A rank generator formalises the idea of generating a rank: the ranks it generates
	must be invariant to scale and only dependent on the ordering of elements.
	The standard ranking functions from statistics such as fractional ranking and
	dense ranking fit into this framework.

	\begin{theorem}\label{prop:rho-max-mean}
		Let $\{R_1,R_2,\dots,R_d\}$ be a set of rankings with common domain $\Omega$
		and $\sigma$ be a rank generator.
		Then
		\[
			\arg\max_{R \in \codom \sigma} \rho_n(R, R_1,R_2,\dots,R_d)= \sigma\left(\prod_{j=1}^d R_j\right).
		\]
	\end{theorem}

	\begin{proof}
		Consider the expression for Spearman's
		$\rho_n$~\eqref{eq:multi-spearman-emp}:
		\[
			\rho_n(R, R_1,R_2,\dots,R_d) = h(d+1)\left[\frac{2^{d+1}}{n}
			\sum_x \left(R(x) \prod_{j=1}^d R_j(x)\right) - 1
			\right].
		\]
		Focusing on the terms in the sum,
		showing that the best possible $R(x)$ is $\prod_{j=1}^d R_j(x)$
		reduces to showing
		\[
			\sum_{x \in U}\sigma\circ P(x)P(x)
		\]
		is maximal, where $P := \prod_j R_j$.
		Suppose there exists an $P'$ such that
		\[
			\sum_{x \in U}\sigma\circ P'(x)P(x) > \sum_{x \in U}\sigma\circ P(x)P(x).
		\]
		By definition of $\sigma$, there exists a permutation $\xi$ such that
		\begin{align*}
			\sum_{x \in U}\sigma\circ P'(x)P(x) &= \sum_{x \in U}\sigma\circ\xi\circ P(x)P(x) \\
			&= \sum_{x \in U}\xi \circ \sigma\circ P(x)P(x) \\
			&> \sum_{x \in U}\sigma\circ P(x)P(x).
		\end{align*}
		This is a contradiction for any permutation $\xi$ as $\sigma$ is order
		preserving.
	\end{proof}

	\begin{corollary}
		The converse applies, that is:
		\[
			\arg\min_{R \in \codom \sigma} \rho_n(R, R_1,R_2,\dots,R_d)= \sigma\left(\prod_{j=1}^d (1-R_j)\right).
		\]
	\end{corollary}
	\begin{proof}
		Proof follows from a similar argument.
	\end{proof}

	\section{Empirical Copulas with Partial Lists}\label{sec:empirical-copula}

	In many applications it is prohibitive to obtain complete annotations of the
	object ranks.
	For
	example, in the document retrieval setting, this amounts to providing ranks for
	all documents.
	The empirical copula requires the set of rankings $\{R_1,\ldots,R_d\}$ to
	comprise of ranks for the same set of
	elements, that is $\dom R_1 = \cdots = \dom R_d$.
	Hence,
	a key challenge in applying Spearman's $\rho$ to rank aggregation is to estimate
	the
	statistic
	on incompletely labelled lists.

	Recall the definition of the empirical
	copula~\eqref{eq:empirical-copula-2d}.
	We now consider the case where $\dom R \neq \dom S $, but $R$ and $S$ are generated from
	two {\em top ranked\/} lists.
	We define extended rankings $R',S'$ with codomain $[0,1]$ such that $\dom R' = \dom
	R
	\cup \dom S = \dom S'$.
	One way to impute the missing values is to set them to a constant value for all the
	ranks below the
	top-$k$ ranks.
	This value is chosen to be the mid point between the start and end of the missing
	section.
	The values in the top-$k$ are retained to be the original values in the extension.
	The definition below formally defines this notion.
	Note that we have
	to renormalise the values.

	\begin{definition}[non-informative extension]\label{def:noninform}
		Let $R$ be a ranking operator and $R'$ be its extension to domain $\dom R'$.
		Then,
		\begin{align}
			R'(x) = \left\{ \begin{array}{lr}
				\frac{|\dom R|}{|\dom R'|}R(x) & x \in \dom R\\
				\frac{|\dom R|+|\dom R'|}{2|\dom R'|} &\text{otherwise}
			\end{array}\right.\label{eq:ext}
		\end{align}
		$\forall x \in \dom R'$.
	\end{definition}

	We call this the non-informative extension since it assumes that all
	items that are not ranked have the same rank (the mean of the missing
	ranks).
	Note that the two experts $R_i$ and $R_j$ may have ranked
	different numbers of objects.
	An advantage of this extension is that it can easily deal with
	the case of more than two experts.
	Consider $d$ experts
	$R_1,\ldots,R_d$, each of which may have ranked a different subset of
	the objects.
	Hence the extension has to impute values on the union of
	items from all experts.
	Denote $\dom R' := \dom R_1 \cup \ldots \cup
	\dom R_d$, then we can apply \cref{def:noninform} to
	complete each ranking operator $R_j$.
	An additional advantage to the
	non-informative extension is that it results in a consistent ranking.

	\begin{definition}
		An extended ranking $R'$ of $R$ is called {\em consistent\/} if the following
		axioms
		hold:

		\begin{enumerate}
			\item $R'(x) < R'(y)$ $\forall x,y \in \dom R$ with $R(x) < R(y)$
			\item $R'(x) = R'(y)$ $\forall x,y \in \dom R$ with $R(x) = R(y)$
			\item $R'(y) > R'(x)$ $\forall x \in \dom R,\, y \in \dom R'$
		\end{enumerate}

		If $E[R] = E[R']$ also holds, then $R'$ is called {\em strictly consistent}.
	\end{definition}

	\begin{lemma}
		\Cref{def:noninform} produces a consistent ranking.
		If $E[R]=\frac{1}{2}$ then~\eqref{eq:ext} produces a strictly
		consistent ranking.
	\end{lemma}

	\begin{proof}
		The notation $|\dom R|$ can become unwieldly in following proof.
		We therefore adopt the shorthand notations $r:=|\dom R|$
		and $r':=|\dom R'|$ for the size of the respective sets.

		Axioms 1 and 2 are satisfied by definition as
		the map $x \mapsto \frac{r}{r'}x$ is monotonic.
		For all $x \in \dom R' \setminus \dom R$,
		\begin{equation*}
			R'(x) = \frac{r+r'}{2r'}
			\leqslant 2R(y)\frac{r+r'}{2r'}
			\leqslant \frac{2R(y)r}{2r'}
			= \frac{r}{r'}R(y)
			= R'(y)
		\end{equation*}
		for any $y \in \dom R$, satisfying axiom 3.

		Furthermore, as
		\begin{align*}
			E[R'] &= \frac{1}{r'}\left(\sum_{x \in \dom R} R'(x) + \sum_{x
			  \in \dom R' \setminus \dom R} R'(x)\right) \\
			&= \frac{1}{r'}\left(\frac{r}{r'}\sum_{x \in \dom R} R(x) + (r' - r)\frac{r+r'}{2r'}\right) \\
			&= \frac{1}{r'}\left(\frac{r^2}{r'}E[R] + (r'-r)\frac{r+r'}{2r'}\right) \\
			&= \frac{r^2(2E[R]-1)+r'}{2r'^2},
		\end{align*}
		$R'$ is strictly consistent if $E[R] = \frac{1}{2}$.
	\end{proof}

	\Cref{def:noninform} is called a {\em non-informative\/} extension as it uses
	no additional information and does not bias the imputed elements in anyway:
	imputed values are all considered tied and mapped to the same value.
	Furthermore, the strictly consistent property that \cref{def:noninform}
	satisfied is important
	when using fractional ranking as it guarantees no introduction of bias.

	Note also that there is a dual imputation whereby missing values are assigned to
	the top of the list rather than the bottom.
	This is equivalent to the above
	imputation applied to reverse rankings.
	The choice of top or bottom
	imputation is application dependent.

	\subsection{Empirical Upper and Lower Bounds}

	\begin{proposition}
		For top-$k$ lists where $k$ of $n$ items are ranked by all $d$ experts
		with codomain $\{1,\dots,n\}$ (i.e., unnormalised ranks), the
		Spearman's $\rho$ is bounded by
		\[
			\rho_n(R_1,R_2,\dots,R_d) = \rho_k(R_1,R_2,\dots,R_d) + C,
		\]
		where
		\begin{gather*}
			\frac{{2}^{d}\,\mathrm{h}\left( d\right) \,\left( \left( k\,{\left( k+1\right)
			    }^{d}-n\,{\left( n+1\right) }^{d}\right) \,\left(
			  \sum_{i=1}^{k}\prod_{j=1}^{d}\frac{{R}_{j}\left( i\right) }{k+1}\right)
			  +k\,\sum_{i=k+1}^{n}{i}^{\frac{d}{2}}\,{\left( k-i+n+1\right)
			    }^{\frac{d}{2}}\right) }{n\,{\left( n+1\right) }^{d}\,k}\\
			\leq C \leq\\
			\frac{{2}^{d}\,\mathrm{h}\left( d\right) \,\left( \left( k\,{\left(
			    k+1\right) }^{d}-n\,{\left( n+1\right) }^{d}\right) \,\left(
			  \sum_{i=1}^{k}\prod_{j=1}^{d}\frac{{R}_{j}\left( i\right)
			    }{k+1}\right) +\left( \sum_{i=k+1}^{n}{i}^{d}\right) \,k\right)
			  }{n\,{\left( n+1\right) }^{d}\,k}.
		\end{gather*}
	\end{proposition}
	\begin{proof}
		Proof sketch: the definition of $\rho$ for unnormalised rankings is
		\[
			\rho_n(R_1,R_2,\dots,R_d) = h(d+1)\left[\frac{2^{d+1}}{n} \sum_{i=1}^n
			\left(\prod_{j=1}^d \frac{R_j(i)}{n+1}\right) - 1.
			\right]
		\]
		By considering the difference $\rho_n(R_1,R_2,\dots,R_d) -
		\rho_k(R_1,R_2,\dots,R_d)$
		and factorising out the common terms, we obtain
		\[
			C = \frac{\left( d+1\right) \,{2}^{d}\,\left( k\,\left(
			  \sum_{i=1}^{n}\prod_{j=1}^{d}\frac{{R}_{j}\left( i\right)
			    }{n+1}\right) -\left( \sum_{i=1}^{k}\prod_{j=1}^{d}\frac{{R}_{j}\left(
			    i\right) }{k+1}\right) \,n\right) }{\left( {2}^{d}-d-1\right)
			  \,k\,n}.
		\]
		The term $\sum_{i=1}^{n}\prod_{j=1}^{d}\frac{{R}_{j}\left( i\right)
		  }{n+1}$
		can be bounded above by
		\[
			\sum_{i=1}^{n}\prod_{j=1}^{d}\frac{{R}_{j}\left( i\right) }{n+1} =
			\sum_{i=1}^{k}\prod_{j=1}^{d}\frac{{R}_{j}\left( i\right) }{n+1}+
			\sum_{i=1+k}^{n}\prod_{j=1}^{d}\frac{{R}_{j}\left( i\right) }{n+1}
			\leq
			\sum_{i=1}^{k}\prod_{j=1}^{d}\frac{{R}_{j}\left( i\right) }{n+1}+
			\sum_{i=1+k}^{n}{\left( \frac{i}{n+1}\right)}^d,
		\]
		and below by
		\begin{align*}
			\sum_{i=1}^{n}\prod_{j=1}^{d}\frac{{R}_{j}\left( i\right) }{n+1} &=
			\sum_{i=1}^{k}\prod_{j=1}^{d}\frac{{R}_{j}\left( i\right) }{n+1}+
			\sum_{i=1+k}^{n}\prod_{j=1}^{d}\frac{{R}_{j}\left( i\right) }{n+1}\\
			&\geq
			\sum_{i=1}^{k}\prod_{j=1}^{d}\frac{{R}_{j}\left( i\right) }{n+1} +
			\sum_{i=k+1}^{n}\left(
			\prod_{j=i}^{\lceil\frac{d}{2}\rceil}\frac{i}{n+1}\right)
			\,\prod_{j=\lceil\frac{d}{2}\rceil}^d\frac{n-i+1+k}{n+1},
		\end{align*}
		giving us the bounds in the proposition.
	\end{proof}

	\subsection{Optimal Imputation}\label{sec:optimpute}

	An alternative to the previously presented imputation method is to impute such
	that $\rho$ is maximised or minimised.
	In general this is a NP-hard problem as it
	involves searching all permutations.
	In this section, we formulate
	this as an optimisation problem.

	Let $\mathbb{I} = \{1,\dots,n\} \times \{1,\dots,d\}$ be indices over $n$ items
	and $d$ experts.
	Let $\mathbb{O} \subset \mathbb{I}$ be the observed indices
	(for which we have a rank) and define $\mathbb{U} := \mathbb{I} \setminus
	\mathbb{O}$.
	We then have a rank function $R \colon \mathbb{O} \to
	\{1,\dots,n\}$.
	Recall that Spearman's $\rho$ is determined by a sum of the products over
	ranks.
	By introducing a log transformation, we convert the product
	into a sum using the logarithm rule:
	\[
		\sum_{i=1}^n \left(\prod_{j=1}^d \frac{R_j(i)}{n+1}\right) =
		\sum_{i=1}^n \left( \exp \log \prod_{j=1}^d \frac{R_j(i)}{n+1}
		\right)
		= \sum_{i=1}^n \left( \exp \sum_{j=1}^d \log \frac{R_j(i)}{n+1}
		\right).
	\]

	\subsubsection{Imputing to Maximise Correlation}
	We can maximise Spearman's $\rho$ by introducing binary indicators $x_{i,j,k}$
	indexed over $\mathbb{I} \times \{1,\dots,n\}$ to denote a rank of $k$ for
	item $i$ in list $j$.

	\[ \max_{x_{i,j,k}} \sum_{i=1}^n\exp\left[ \sum_{j=1}^d \sum_{k=1}^n
	x_{i,j,k}\log\left(\frac{k}{n+1}\right)\right] \]
	such that
	\begin{align}
		\sum_k x_{i,j,k} &= 1 &\forall i,j\label{peritem}\\
		\sum_i x_{i,j,k} &= 1 &\forall k,j\label{perlist}\\
		\sum_{k} x_{i,j,k}k &= R(i,j) &\forall (i,j) \in \mathbb{O}\label{matchknown}\\
		x_{i,j,k}&\in \{0,1\} &\forall i,j,k
	\end{align}

	Constraint~\eqref{peritem} ensures an item is only assigned one rank per expert,
	and constraint~\eqref{perlist} ensures a rank is only assigned once per expert.
	Finally, the third constraint~\eqref{matchknown} ensures known ranks are
	assigned.

	\subsubsection{Imputing to Minimise Correlation}

	Analogously, we can consider the problem of minimising Spearman's $\rho$.
	\begin{equation}
		\label{eq:imputation-min}
		\min_{x_{i,j,k}} \sum_{i=1}^n\exp\left[ \sum_{j=1}^d \sum_{k=1}^n x_{i,j,k}\log\left(\frac{k}{n+1}\right)\right]
	\end{equation}
	such that
	\begin{align*}
		\sum_k x_{i,j,k} &= 1 &\forall i,j\\
		\sum_i x_{i,j,k} &= 1 &\forall k,j\\
		\sum_{k} x_{i,j,k}k &= R(i,j) &\forall (i,j) \in \mathbb{O}\\
		x_{i,j,k}&\in \{0,1\} &\forall i,j,k
	\end{align*}
	By considering the relaxation of $x_{i,j,k}\in \{0,1\}$ to
	$x_{i,j,k}\in [0,1]$, we obtain a convex optimisation problem.

	\begin{proposition}
		The relaxation of optimisation problem~\eqref{eq:imputation-min} such that
		$x_{i,j,k}$ is in the interval $[0,1]$ is a convex
		optimisation problem.
	\end{proposition}

	\begin{proof}
		The objective has the form $\sum_i \exp ( \langle x_i, \omega
		\rangle)$ with $\omega \in
		{[\log\left(\frac{1}{n+1}\right), \dots,
		  \log\left(\frac{n}{n+1}\right)]}^{nd}$.
		Thus, as each term in the sum is convex, and as the sum of convex functions is
		convex, the objective is convex.
		The constraints are all linear,
		hence this is a convex optimisation problem.
	\end{proof}

	\cheng{Is the solution of the relaxed minimisation problem always integer?}

	\cheng{The objective is likely to be submodular since the relaxed
	  version is convex.
	  Unfortunately, this still doesn't give us an easy
	  discrete optimisation problem.}

	However, as a consequence of \cref{th:minW}, as $d\to\infty$ we know
	that $\rho\geq 0$, hence the minimum $\rho$ will approach the $\rho$
	when using the non-informative extension (the non-informative
	extension has $\rho=0$), thus there is little need to solve the
	optimisation problem after a sufficient number of dimensions is
	reached.

	\subsection{Experiments on University Ranking}

	The optimal imputation algorithm presented in \cref{sec:optimpute} is
	difficult to
	solve due to the integer constraints.
	We evaluated the performance of a relaxed
	version of the program, whereby the constraints are relaxed such that the
	variables may take a value in the range $[0,1]$.
	To solve the relaxed problem,
	we used a BFGS based optimiser by shifting the equality constraints into the
	objective function with high penalties.
	Final ranks were determined by ranking
	item $i$ in list $j$ based on the score $\sum_{k} x_{i,j,k}k$.

	We evaluated this relaxed solution on imputing rankings for universities.
	To
	this end, the top-200 universities ranked by QS in 2014, Shanghai in 2014, and
	Times in 2015 were obtained.
	In aggregating these three lists, there are a total
	of 266 ranks that need to be imputed.

	Measuring multivariate Spearman's $\rho$ on all three lists imputing the missing
	elements using the non-informative extension gives $\rho = 0.632$.
	In
	comparison, the relaxed optimal imputation found a solution that obtained $\rho
	= 0.683$, a modest increase in the correlation.
	We also developed an
	interactive website\footnote{\url{http://uni.cua0.org}} showing the
	detailed
	results for all universities, which also allows the user to alter the
	weights of each of the original experts.
	The top 36 aggregate rankings for the
	universities are given in~\cref{uniranks}.

	\section{Supervised Learning to Rank}\label{sec:benchmarks}

	We now consider the task of learning rank aggregation from extreme ranks.
	\Cref{prop:rho-max-mean} and \cref{def:noninform} provide the core of
	our algorithm.
	Using \cref{prop:rho-max-mean}, we can find an “average” rank that
	aggregates a set of ranks, and by extending top-$k$ and bottom-$k$
	ranks to a common domain, we can apply it to partially labelled data.

	\subsection{Weighted Mixture of Experts}

	As a result of \cref{prop:rho-max-mean} we have a way of finding the ranking
	(according to some rank generator) that is closest to a set of ranks.
	Consider the learning problem where we have a ranking $L$ which
	comprise our labels, and a set of $d$ experts $\{R_j\}$.
	During training, we would like to find a weighting of the input
	rankings such that it gives the label.
	Given a target ranking $L$,
	we would like to optimise the weights $\omega$,
	\[
		\max_{\omega} \rho_n(L, R_1^{\omega_1}, R_2^{\omega_1}, \ldots,
		R_n^{\omega_d}).
	\]
	Here we have introduced weights $\omega$ over each rank to control the influence of
	each rank over the final consensus rank; the intuition here is that ranks with
	$\omega_i>1$ are “replicated” with more influence, which is easy to see when
	$\omega_i$ are natural numbers.
	For example, a weight of $2$ would mean the ranked list has appeared twice in the
	calculation of the consensus rank.
	While it is convenient to have
	integer weights for interpretability, the weights $\omega$ could be
	any real number in general.
	In the following, we consider
	$\omega\in\RR^n$.
	Instead of performing this high-dimensional optimisation, we
	decompose it into a pairwise (bivariate) comparison between the label
	$L$ and the weighted geometric mean, where we now explicitly show the
	fact that the ranks are a function of the $n$ objects $x$
	\[
		\max_{\omega} \sum_{x} \rho_n(L(x),\, \sigma(R_1^{\omega_1}\otimes
		R_{2}^{\omega_2}\otimes\dots\otimes R_{d}^{\omega_n})(x)),
	\]
	where the notation $\otimes$ indicates the product
	operator.
	Observe that we have used \cref{prop:rho-max-mean}
	to convert the $d$ dimensional problem into the product of
	ranks $R_j$ and the Spearman's correlation above is only two dimensional.
	For bivariate Spearman's $\rho$, this can be expressed in terms of the
	squared difference~\eqref{eq:emp-spearman}.
	We further assume that $\sigma$ is the identity mapping to simplify the
	problem, giving us:
	\begin{equation}\label{eq:opt}
		\min_{\omega} \sum_{x} {\left(L(x) - R_1^{\omega_1}(x)R_{2}^{\omega_2}(x)\dots R_{d}^{\omega_n}(x)\right)}^2.
	\end{equation}
	The objective~\eqref{eq:opt} minimises the distance between the label ranks and
	the weighted expert ranks.

	\subsection{Least Squares Method on Logarithm of Ranks}

	Recall that we consider normalised ranks (divided by
	$n+1$).
	By using the logarithm identity, we
	convert the power scaling in~\eqref{eq:opt} into a
	multiplicative scaling.
	Our algorithm is:
	\begin{enumerate}
		\item Extend incomplete ranks $\{R_i\}$ to $\{R'_i\}$ by imputing
		the average missing value
		such that $\dom R'_i = \dom L$;
		\item Convert to log-ranks $r'_i = \log\circ R'_i$ and $l = \log \circ L$;
		\item Learn weights $\omega$ by minimising
		\[
			\sum_{x}{\left(l(x) - \sum_{j=1}^d\omega_j r'_j(x)\right)}^2,
			\qquad\mbox{where the outer sum is over the $n$ examples $x$.}
		\]
	\end{enumerate}

	A log transformation
	of the ranks is used as it naturally encodes the weights as a power
	scaling in the framework of \cref{prop:rho-max-mean}, i.e., the
	weighted consensus rank is given by $\prod_i{r'_i(x)}^{\omega_i}$.
	Note that this is still solving~\eqref{eq:opt} as we are optimising Spearman's
	$\rho$, which is sensitive only to ordering, and therefore though the final
	weights
	are different $\rho$ is maximised via~\eqref{eq:emp-spearman}.

	In the following experiments we also included a bias/offset term in the
	least squares problem, which can be interpreted as adding a ranking
	that is constant (gives all objects the same rank).
	It is
	interesting to note that the final step in this procedure is closely related to
	Borda
	Count, except our consensus rank is the geometric mean instead of the
	arithmetic mean.
	Since this is a least squares estimation problem, we directly use the closed form
	solution.

	\subsection{Benchmarking on LETOR 4.0}

	\begin{figure}[tbp]
		\subfloat[MQ2007]{\includegraphics[width=0.49\textwidth]{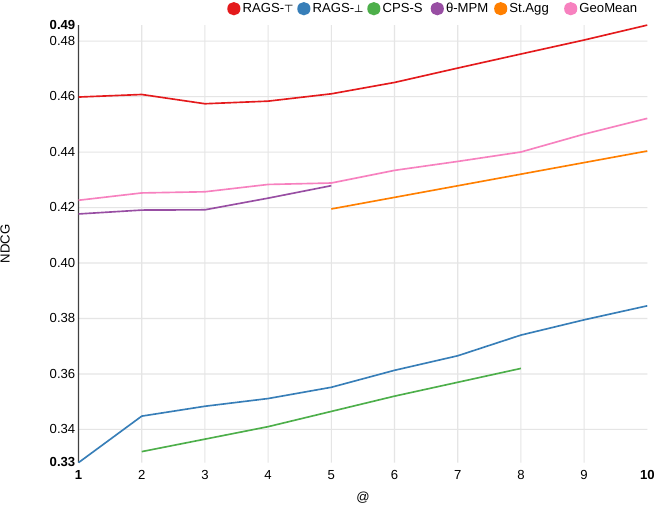}\label{fig:mq-agg:2007}}
		\subfloat[MQ2008]{\includegraphics[width=0.49\textwidth]{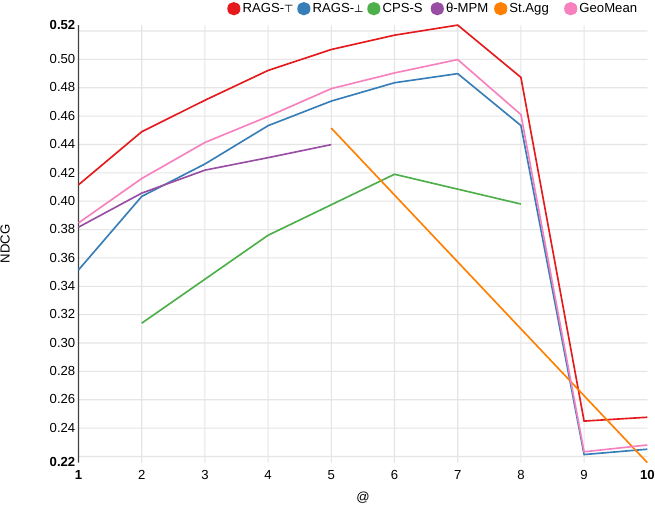}\label{fig:mq-agg:2008}}
		\caption{Results on MQ2007-agg (a, left) and MQ2008-agg (b, right): NDCG@k.
		  Our method is labelled RAGS-⊤ and RAGS-⊥ corresponding to top and bottom
		  non-informative imputation respectively.
		  The results for CPS-S was the best reported in~\citet{qin10promra}.
		  The results of $\theta$-MPM was the
		  best among the reported results in~\citet{volkovs12flegmp} from BordaCount,
		  CPS,
		  SVP, Bradley-Terry model, and Plackett-Luce model.
		  The results of St.Agg was the best among the
		  reported results in~\citet{niu13stora} and was the best among MCLK, SVP,
		  Plackett-Luce model, $\theta$-MPM,
		  BordaCount and RRF\@.
		  }\label{fig:mq-agg}
	\end{figure}

	We tested our method on the MQ2007-agg and MQ2008-agg list aggregation
	benchmarks~\cite{qin10letor}.
	The goal in these challenges is to aggregate 21 and 25 different rankers
	respectively over a set of
	query-document pairs.
	Each data set has 5 pre-defined cross-validation folds with each fold providing a
	training, testing and validation data set (60\%/20\%/20\%).
	We trained our model on the training set and tested on the testing
	set, leaving the validation set unused since we have no hyperparameters.

	In the following we consider two types of experts:
	either experts $\{R_j\}$ are top-$k$ experts, that is they only rank the best $k$
	samples from
	$\Omega$, or experts are bottom-$k$ experts, that is they identify the
	worst $k$ samples from $\Omega$.
	We call our proposed method RAGS-⊤ and
	RAGS-⊥ respectively.
	We assume that the ranked documents in the
	benchmark data sets are either top-$k$ or bottom-$k$ respectively,
	with potentially different numbers of documents $k$ labelled by each
	expert.
	Ties are given the average rank of tied documents.

	To evaluate the agreement, we use the standard evaluation tool from the LETOR
	website\footnote{\url{http://research.microsoft.com/letor}}, which
	implements the Normalised Discounted
	Cumulative Gain (NDCG).
	In \cref{fig:mq-agg:2007}, we see that our approach RAGS-⊤ performs better than
	all other methods at any selection size on the MQ2007-agg data set.
	Indeed, we also perform better than \citet{qin10promra} where the best result
	uses a coset-permutation
	distance based stagewise (CPS) model with Spearman's $\rho$ in a probabilistic
	model.
	Recall that
	our approach considers the multivariate Spearman's $\rho$ whereas
	\citet{qin10promra}
	uses bivariate Spearman's $\rho$ in a pairwise fashion.
	For MQ2008-agg (\cref{fig:mq-agg:2008}), again our approach performs better
	than all other methods.

	\begin{sidewaystable*}[p]
		\centering
		\caption{Results on MQ2007-agg: NDCG\@.
		  Our method is labelled RAGS-⊤ and RAGS-⊥ corresponding to top and bottom
		  non informative imputation respectively.
		  The results for CPS-S was the best reported in \citet{qin10promra}.
		  The results of $\theta$-MPM was the
		  best among the reported results in \citet{volkovs12flegmp} from BordaCount,
		  CPS,
		  SVP, Bradley-Terry model, and Plackett-Luce model.
		  The results of St.Agg was the best among the
		  reported results in \citet{niu13stora} and was the best among MCLK, SVP,
		  Plackett-Luce model, $\theta$-MPM,
		  BordaCount and RRF.}\label{tab:mq2007:ndcg}
		\begin{tabular}{ccccccccccc}
			\toprule
			Fold & @1 & @2 & @3 & @4 & @5 & @6 & @7 & @8 & @9 & @10 \\
			\midrule
			RAGS-⊤ & 0.45986 & 0.46078 & 0.45744 & 0.45838 & 0.46102 & 0.46512 & 0.4703 & 0.47538 &
			0.48042 & 0.4858 \\
			RAGS-⊥ & 0.32804 & 0.3448 & 0.34836 & 0.35114 & 0.3552 & 0.36132 & 0.36656 & 0.37402 &
			0.37952 & 0.38458 \\
			\midrule
			CPS-S & & 0.332 & & 0.341 & & 0.352 & & 0.362 & & \\
			θ-MPM & 0.4177 & 0.4191 & 0.4192 & 0.4234 & 0.4279 & & & & & \\
			St.Agg & & & & & 0.4195 & & & & & 0.4404 \\
			GeoMean & 0.42264 & 0.42528 & 0.42570 & 0.42834 & 0.42886 & 0.43342 & 0.43664 &
			0.44004 & 0.44648 & 0.45216 \\
			\bottomrule
		\end{tabular}
	\end{sidewaystable*}

	\begin{sidewaystable*}
		\centering
		\caption{Results on MQ2008-agg: NDCG}\label{tab:mq2008:ndcg}
		\begin{tabular}{cccccccccccc}
			\toprule
			Fold & @1 & @2 & @3 & @4 & @5 & @6 & @7 & @8 & @9 & @10 \\
			\midrule
			RAGS-⊤ & 0.41158 & 0.44898 & 0.47118 & 0.4922 & 0.50696 & 0.51706 & 0.52416 & 0.48732 &
			0.24498 & 0.24768 \\
			RAGS-⊥ & 0.35156 & 0.40338 & 0.42624 & 0.45326 & 0.4706 & 0.48352 & 0.48994 & 0.45336 &
			0.22138 & 0.22514 \\
			\midrule
			CPS-S & & 0.314 & & 0.376 & & 0.419 & & 0.398 & & \\
			θ-MPM & 0.3817 & 0.4057 & 0.4219 & 0.4307 & 0.4399 & & & & & \\
			St.Agg & & & & & 0.4515 & & & & & 0.2157 \\
			GeoMean & 0.38470 & 0.41600 & 0.44142 & 0.45976 & 0.47938 & 0.49042 & 0.49986 &
			0.46108 & 0.22334 & 0.22812 \\
			\bottomrule
		\end{tabular}
	\end{sidewaystable*}

	To tease apart the effect of imputing missing ranks and the effect of weighting the
	experts, we
	compared our proposed method with and without training (uniform
	weights).
	GeoMean denotes the results for the geometric mean (uniform weights on
	the experts) after performing imputation assuming top-$k$ ranking by
	the experts.
	First we observe that our proposed approach outperforms the geometric mean,
	which is a good sanity check.
	It is surprising that the geometric mean performs quite well in MQ2007.
	The major
	difference is that we are imputing the missing ranks, and the other methods suffer
	from assigning
	them to an arbitrary value.
	This demonstrates the importance of imputation.

	\subsection{Strictly Ordered Labels}

	One issue with the benchmark aggregation data set is that the labels are only
	\{0,1,2\} relevance
	scores, and hence it is unclear exactly what the rankings are within the relevance
	classes.
	We create a new data set which is formed by taking the intersection
	between the documents retrieved by a particular query between
	MQ2007-agg and MQ2007-list.
	This new data set contains the strictly
	ordered labels from MQ2007-list, but uses the aggregation data from MQ2007-agg.
	The same procedure
	is used to create the corresponding data set for MQ2008-agg and MQ2008-list.
	These data sets are
	available for download at the LETOR website.
	We maintain exactly the same 5-fold cross
	validation splits and report our results in~\cref{tab:mq-agglist}.

	Considering the results for Spearman's $\rho$, we observe that our learning
	method
	performs well.
	Note that the geometric mean outperforms Borda count on both data sets,
	which confirms that our theoretically justified model performs better
	than the heuristic model.
	It is interesting to observe that optimising for Spearman's $\rho$ could
	result in a decrease in Kendall's $\tau$.
	This demonstrates the importance
	of choosing the appropriate objective function for learning.

	\begin{table}[ht]
		\centering
		\caption{Results on MQ2007-agglist and MQ2008-agglist.
		  The left
		  column shows the results for multivariate Spearman's $\rho$ and the right column
		  shows
		  the result
		  for Kendall's $\tau$.}\label{tab:mq-agglist}
		\begin{tabular}{ccccc}
			\toprule
			& \multicolumn{2}{c}{MQ2007-agglist} &
			\multicolumn{2}{c}{MQ2008-agglist}\\
			\cmidrule(r){2-3}
			\cmidrule(r){4-5}
			Method & $\rho$ & $\tau$ & $\rho$ & $\tau$\\
			\midrule
			RAGS-⊤ & 0.4394 & 0.6201 & 0.7235 & 0.6931\\
			RAGS-⊥ & 0.2992 & 0.2488 & 0.6349 & 0.5560\\
			\midrule
			GeoMean & 0.2457 & 0.3011 & 0.5777 & 0.6578\\
			Borda & 0.2217 & 0.1790 & 0.5519 & 0.5869\\
			\bottomrule
		\end{tabular}
	\end{table}

	\section{Discussion and Conclusion}

	We propose an approach for learning weights between experts for the
	task of rank aggregation.
	By generalising the derivation of concordance functions, we obtain an
	expression for multivariate Spearman's $\rho$.
	Furthermore, we show that the
	geometric mean of the expert ranks is the optimal aggregator under
	Spearman's correlation.
	Motivated by this, our method solves a least squares estimation problem
	for logarithmic normalised ranks to find optimal weights.

	One possible extension of our work is to compute the correlation for all
	possible subsets of rankings.
	While \cref{th:minW} shows that the overall
	correlation cannot be negative as the number of rankings increase, there may
	be subgroups which are positively correlated within groups but negatively
	correlated between groups.
	By computing the correlation on the power set,
	we could use a clustering method to find such subgroups.

	Though we have focused on $\rho_d^+$, our results are equally applicable to
	$\rho_d^-$; indeed it is a simple reversal of ranks that give $\rho_d^-$.
	The choice between $\rho_d^+$ and $\rho_d^-$ is thus problem dependent: for
	tasks where being ranked highly is more informative $\rho_d^+$ is a better
	choice; conversely $\rho_d^-$ is more suitable for tasks where being ranked
	lowly is more informative.

	In contrast to other rank aggregation approaches, our method is very
	computationally efficient.
	However, the core of our method requires a complete set of rankings and hence does
	not
	handle missing variables.
	To resolve this, we propose three imputation methods (unbiased, optimistic,
	pessimistic) for
	completing top-$k$ ranked lists that allows us to apply Spearman's $\rho$ to
	aggregate
	ranks from
	partial lists.
	Our method is thus applicable for large scale applications with top-$k$ rankings
	that arise in areas such as text mining and bioinformatics.
	One subtlety is that
	imputation from top-$k$ should not be confused with the choice of using
	$\rho_d^+$,
	which is an separate design choice.

	Surprisingly, our weighted geometric mean shows state of the art results on
	benchmark
	data sets, without the need for tuning hyperparameters or expensive computation.
	The simplicity of our model makes it easier to interpret, and the weights
	give a direct estimate of the influence of each expert.
	This problem has wide applications to ensemble learning, voting,
	text mining, recommender systems and bioinformatics.

	\acks{This work was completed when both authors were employed by NICTA\@.
	  NICTA was funded by the Australian Government through the Department of
	  Communications and the Australian Research Council through the ICT Centre of
	  Excellence Program.}

	\newpage
	\appendix
	\section{Bivariate Spearman's $\rho$ and Squared Distance}

	This well known
	result\footnote{\url{http://en.wikipedia.org/wiki/Spearman's_rank_correlation_coefficient}
	  accessed on 20 May 2014} shows that Spearman's $\rho$ can be expressed in
	terms of the squared distance between ranks.

	In the following derivation, we use the expressions for the sum of
	integers and the sum of squares of integers:
	\[
		\sum_{k=1}^n k_i = \frac{n(n+1)}{2}\qquad \sum_{k=1}^n k_i^2 = \frac{n(n+1)(2n+1)}{6}.
	\]
	Recall that Spearman's $\rho$ is defined~\eqref{eq:emp-spearman} as:
	\[
		\rho_n = \frac{\sum_{x} (R(x) - \bar{R}) (S(x) - \bar{S}) }
		{\sqrt{\sum_{x} {(R(x) - \bar{R})}^2 \sum_{x} {(S(x) - \bar{S})}^2}}.
	\]
	Since there are no ties, both $R(x)$ and $S(x)$ consist of integers
	from 1 to $n$ inclusive, and the two squared sums in the denominator
	are the same.
	Recall that the mean rank is
	\[
		\bar{R} = \bar{S} = \frac{n+1}{2},\quad\mbox{and}\quad
		\sum_x R(x) = \frac{n(n+1)}{2} = n\bar{R}.
	\]

	Therefore, the denominator can be expressed as a function of $n$:
	\begin{align*}
		\sqrt{\sum_{x} {(R(x) - \bar{R})}^2 \sum_{x} {(S(x) - \bar{S})}^2}
		&= \sum_{x} {(R(x) - \bar{R})}^2\\
		&= \sum_x ({R(x)}^2 - 2 R(x)\bar{R} + \bar{R}^2)\\
		&= \sum_x {R(x)}^2 - 2\bar{R}\sum_x R(x) + n\bar{R}^2\\
		&= \sum_x {R(x)}^2 - n\bar{R}^2\\
		&= \frac{n(n+1)(2n+1)}{6} - n{\left(\frac{n+1}{2}\right)}^2\\
		&= n(n+1)\left( \frac{2n+1}{6} - \frac{n+1}{4} \right)\\
		&= n(n+1)\left( \frac{n-1}{12} \right)\\
		&= \frac{n(n^2-1)}{12}.
	\end{align*}

	Since both $R(x)$ and $S(x)$ consists of the same integers, we can
	express the squared difference in terms of the product.
	\begin{align*}
		\sum_x \frac{1}{2} {(R(x) - S(x))}^2
		&= \sum_x \frac{1}{2} ({R(x)}^2 -2R(x)S(x) + {S(x)}^2)\\
		&= \sum_x \frac{1}{2} ({R(x)}^2 + {S(x)}^2) - \sum_x R(x)S(x)\\
		&= \sum_x {R(x)}^2 - \sum_x R(x)S(x),
	\end{align*}
	where the first term is a function of $n$.

	We express the product of the means, which appears in the numerator
	later, to match the denominator:
	\begin{align*}
		n{\left(\frac{n+1}{2}\right)}^2
		&= \frac{n(n+1)}{12} 3(n+1)\\
		&= - \frac{n(n+1)}{12} [(n-1) - (4n+2)]\\
		&= - \frac{n(n+1)(n-1)}{12} + \frac{n(n+1)(2n+1)}{6}\\
		&= - \frac{n(n^2-1)}{12} + \sum_x {R(x)}^2,
	\end{align*}
	where the last term in the sum is the expression for the sum of squares.
	We can now derive the expression for the numerator:
	\begin{align*}
		\sum_x (R(x) - \bar{R})(S(x)-\bar{S})
		&= \sum_x R(x)S(x) - \bar{R}\sum_x S(x) - \bar{S}\sum_x R(x) +
		n\bar{R}\bar{S}\\
		&= \sum_x R(x)S(x) - n\bar{R}\bar{S}\\
		&= \sum_x R(x)S(x) - n{\left(\frac{n+1}{2}\right)}^2\\
		&= \sum_x R(x)S(x) + \frac{n(n^2-1)}{12}- \sum_x {R(x)}^2\\
		&= \frac{n(n^2-1)}{12} - \sum_x \frac{1}{2} {(R(x) - S(x))}^2,
	\end{align*}
	where the last line uses the expression of the sum of squared
	differences above.
	Putting together the expressions for the numerator
	and denominator together gives the desired result:
	\[
		\rho_n = 1 - \frac{6\sum_x {(R(x) - S(x))}^2}{n(n^2-1)}.
	\]

	\clearpage

	\section{University Aggregate Rankings}\label{uniranks}

	\begin{center}
		\begin{tabular}{cc}
			\toprule%
			Rank & University\\
			\midrule%
			1 & Harvard University\\
			2 & Massachusetts Institute of Technology\\
			3 & Stanford University\\
			4 & California Institute of Technology\\
			5 & University of Cambridge\\
			6 & University of Oxford\\
			7 & Princeton University\\
			8 & University of Chicago\\
			9 & University of California, Berkeley\\
			10 & Imperial College London\\
			11 & ETH Zurich\\
			12 & University College London\\
			13 & Yale University\\
			14 & Columbia University\\
			15 & Johns Hopkins University\\
			16 & Cornell University\\
			17 & University of California, Los Angeles\\
			18 & University of Pennsylvania\\
			19 & University of Michigan\\
			20 & University of Toronto\\
			21 & Duke University\\
			22 & Northwestern University\\
			23 & University of Edinburgh\\
			24 & University of California, San Diego\\
			25 & King's College London\\
			26 & University of Washington\\
			27 & University of Tokyo\\
			28 & National University of Singapore\\
			29 & New York University\\
			30 & \'Ecole Polytechnique F\'ed\'erale de Lausanne\\
			31 & McGill University\\
			32 & University of Melbourne\\
			33 & University of Illinois at Urbana--Champaign\\
			34 & University of Wisconsin--Madison\\
			35 & University of British Columbia\\
			36 & University of Manchester\\
			37 & Australian National University\\
			\bottomrule%
		\end{tabular}
	\end{center}

	\clearpage

	\bibliography{../techreport/copula}

\end{document}